\documentclass{svproc}

\usepackage{tikz}
\usetikzlibrary{decorations.pathreplacing,calligraphy}
\usepackage{pdfpages}
\usepackage{url}
\usepackage{bm}
\usepackage{amsmath}
\usepackage{amssymb}
\usepackage{amsfonts}
\usepackage{todonotes}
\usepackage{hyperref}
\usepackage{cleveref}
\usepackage{caption,subcaption}

\def\zf{\bm z}
\def\R{\mathbb R}
\def\E{\mathbb E}

\def\Laplace{\mathrm{\Delta}}
\def\Fc{\mathcal{F}}

\def\ifat{\bm i}

\def\kfat{\bm k}

\def\lin{{l_{\mathrm{in}}}}
\def\lhidden{{l_{\mathrm{hid}}}}
\def\lout{{l_{\mathrm{out}}}}

\begin{document}
\mainmatter
\title{Solving Partial Differential Equations with Equivariant Extreme Learning Machines}
\titlerunning{Extreme Learning Machines for PDEs}  %

\author{Hans Harder$^1$ \and Jean Rabault$^2$ \and Ricardo Vinuesa$^3$ \and Mikael Mortensen$^4$ \and Sebastian Peitz$^1$}
\authorrunning{Harder et al.}
\institute{Department of Computer Science, Paderborn University, Paderborn, Germany \\
Contact: \email{hans.harder@uni-paderborn.de}
\and
Independent Researcher, Oslo, Norway
\and 
FLOW, Engineering Mechanics, KTH, Stockholm, Sweden %
\and 
Department of Mathematics, University of Oslo, Oslo, Norway
}

\maketitle              %

\begin{abstract}
We utilize extreme learning machines for the prediction of partial differential equations (PDEs). Our method splits the computational domain into multiple windows that are predicted individually using a single model. Despite requiring only few data points (in some cases, our method can learn from a single full-state snapshot), it still achieves high accuracy and can predict the flow of PDEs over long time horizons. Moreover, we show how additional symmetries can be exploited to increase sample efficiency and to enforce equivariance.

\end{abstract}

\setcounter{footnote}{0}

\section{Introduction}
The efficient modeling, prediction, and control of complex systems governed by partial differential equations (PDEs) via data-driven and machine-learning-based techniques has received massive attention in recent years. PDEs are an extremely important class of equations that govern the physics of many applications such as heat transfer, continuum mechanics, fluid dynamics, or quantum mechanics. 

In the past decades, we have seen a plethora of various surrogate-modeling techniques for PDEs. They are useful for a few reasons: 1) in the absence of exact equations or in case of noise, they can learn equations purely from data; 2) surrogate models can often predict larger time steps than conventional solvers, improving computational efficiency; 3) they can automatically identify underlying low-dimensional structures and thereby mitigate the curse of dimensionality \cite{HJE18}; and 4), when implemented in modern deep-learning frameworks, they are differentiable by definition (this is useful for, e.g., control-based applications).
Popular examples are the proper-orthogonal decomposition~\cite{Sir87,KV99,POD19}---where we project the known equations onto a low-dimensional subspace spanned by basis functions learned from data---as well as many non-intrusive methods such as the dynamic-mode decomposition~\cite{Sch10}, i.e., the Koopman operator framework~\cite{RMB+09,NM20,Mau21,PHN+23}.
Quite unsurprisingly, deep learning plays a central role as well, see \cite{WRL19,PTMS22,ZLG22,WP24}. 
More recently, the notion of physics-informed machine learning~\cite{KKLP+21,eivazi2022physics} has become increasingly popular. There---instead of treating the system as a black box---one encodes system knowledge into the training process, for instance in the form of the governing equations \cite{GZ20,LKA+20,LJP+21,RRL+22}. This significantly reduces the amount of required training data, while at the same time improving robustness.
In addition, one may enforce symmetries---also referred to as geometric deep learning~\cite{BBCV21} or group convolutional neural networks~\cite{knigge2022exploiting} in the machine-learning community---to increase efficiency, see, for instance, \cite{TW16,EAMD18} for translational and rotational symmetries in image classification. Using group convolutional neural networks to leverage either invariance or equivariance properties of the problem considered (depending on the underlying problem and the task to accomplish) can result in significant improvement in model performance and robustness, as well as improved data efficiency during the training phase \cite{lafarge2021roto}. In reinforcement-learning-based control of PDEs, the exploitation of translational symmetries gives rise to efficient multi-agent concepts \cite{ZG21,GRS+23,VRV+23,PSC+24}. For complex control tasks where the underlying problem presents strong symmetries which the controller should be aware of, but where these symmetries should be broken to reach a better state, one may obtain significant gains by combining concepts of multi-agent control, group convolutional neural network, and positional encoding \cite{jeon2024advanced}. This results in controllers that are symmetries-aware, but that are also able to choose to actively break existing symmetries if necessary.

When it comes to symmetry exploitation in data-driven modeling of PDEs, mostly translational equivariance has been exploited until now \cite{pathakModelFreePredictionLarge2018,rafayelyanLargeScaleOpticalReservoir2020,pandeyReservoirComputingModel2020,vlachasBackpropagationAlgorithmsReservoir2020}. These works are based on reservoir computing, i.e., recurrent architectures, and it is well-known that these can be challenging to train and require a substantial initialization phase to adjust the latent variables before they can make useful predictions.
The goal of this paper is 1) to increase the sample efficiency while avoiding recurrence, and 2) to include symmetries beyond translation. To this end, we are going to make use of extreme learning machines (ELMs)~\cite{wangReviewExtremeLearning2022,dingExtremeLearningMachine2014,huang2006extreme,DS20} to learn the flow map of chaotic PDEs. ELMs are two-layer networks that can be trained using simple linear regression.
When exploiting translational equivariance similar to \cite{pathakModelFreePredictionLarge2018}, ELMs can be trained from few time steps and make accurate predictions over long time horizons (\Cref{sec:EELMs}). In some instances, a \emph{single} time step suffices to learn a precise approximation of the PDE. Including additional symmetries (i.e., rotations) further increases the efficiency (\Cref{sec:AdditionalSymmetries}).
All simulations have been implemented using the spectral PDE solver \texttt{Shenfun} \cite{mortensenShenfunHighPerformance2018}.
Our \texttt{python} code for the experiments is freely available under \url{https://github.com/graps1/ELM-PDEs}. %

\section{Preliminaries}

Consider a continuous-time dynamical system on some state space $\Fc$,
$$
    \Phi : [0,\infty) \times \Fc \rightarrow \Fc, (t,v) \mapsto \Phi_t (v).
$$
In our setting, $\Fc$ is a function space (e.g., Sobolev space) containing scalar fields of the form $v : \Omega \rightarrow \R$, where $\Omega \subseteq \R^n$ is a set of $n$-dimensional spatial coordinates.
The dynamical system is induced by a PDE of the form
\begin{align}
    \label{eq:pde}
    \partial_t v(t, x) &= f(v(t, x), D v(t, x), D^2 v(t,x), \dots, D^n v(t,x)),
\end{align}
where $v(t) := v(t,\cdot) \in \Fc$ and $D^k v(t,x)$ is a vector containing all spatial derivatives of order $k$ of $v(t)$ evaluated at $x$.
Throughout this text, we have $\Omega = [0,L]^n$, $L > 0$, and assume periodic boundary conditions, that is, $v(t,0) = v(t, x)$ for $x \in \{0,L\}^n$.
The associated flow satisfies
$$
    \Phi_{\Delta t}(v(t)) = v(t+\Delta t),
$$
where $\Delta t \geq 0$ is often fixed and hence omitted. Note that we use $v$ to denote both time-dependent states ($v(t) = v(t,\cdot) \in \Fc$) as well as individual states ($v \in \Fc$). The meaning will be clear from the context.

We study the Cahn-Hillard (CH) equation in two spatial dimensions and three variants of the Kuramoto-Sivashinsky (KS) equation in one or two spatial dimensions; all equations are considered for periodic boundary conditions. 
The Cahn-Hillard equation is given by
\begin{equation}
    \partial_t v = - \Laplace v - \gamma \Laplace^2 v + \Laplace^2 (v^3),
    \label{eq:CH_Nd}
\end{equation}
where $\Laplace$ and $\Laplace^2$ are the Laplace and biharmonic operators, and $\gamma$ is a parameter.
The first variant of the KS equation is defined in two spatial dimensions by
\begin{equation}
    \partial_t v = - \Laplace v - \Laplace^2 v - \textstyle\frac 1 2 | \nabla v |^2,
    \label{eq:KS_Nd}
\end{equation}
where $\nabla$ is the gradient. Another variant is given for one spatial dimension by
\begin{equation}
    \partial_t v = 
    -\partial_{xx}^2 v - \partial_{xxxx}^4 v 
    - (\partial_x v) v,
    \label{eq:KS_1d}
\end{equation}
which is obtained by taking \eqref{eq:KS_Nd} for the case of one spatial dimension, differentiating with respect to $x$ and then substituting $\partial_x v$ by $v$. This form is popular in fluid dynamics, and, due to its chaotic behavior, often used as a test bed for learning algorithms \cite{pathakModelFreePredictionLarge2018,rafayelyanLargeScaleOpticalReservoir2020,vlachasBackpropagationAlgorithmsReservoir2020}. Finally, we consider a variant of \eqref{eq:KS_1d} introduced in \cite{pathakModelFreePredictionLarge2018} that has an explicit dependence on the spatial coordinate:
\begin{equation}
    \partial_t v = 
    -\partial_{xx}^2 v - \partial_{xxxx}^4 v 
    - (\partial_x v) v + \mu \cos\big(2 \pi x / \lambda\big),
    \label{eq:KS_1d_spatial}
\end{equation}
where $\mu$ and $\lambda$ are also parameters.

There are multiple ways to handle the spatial dimensions of a PDE, one common approach being the discretization of space into a grid where every node stores an estimate of the state at that coordinate. 
For a one-dimensional system like \eqref{eq:KS_1d}, the state is approximated by a vector, in the case of two dimensions by a matrix, etc. 
For example, in case of one spatial dimension, one can discretize $\Omega = [0,L]$ into a grid of $m$ evenly distributed nodes by introducing
$$
    \overline v \in \R^m, \quad \overline v_i = v(i \Delta x), \quad\text{for } i = 0,\dots,m-1 \text{ and } \Delta x = L/m.
$$
One then models the flow of $v(t) \in \Fc$ in the discretized space, $\overline v(t) \mapsto \overline v(t+ \Delta t)$. %

To approximate spatial derivatives, one can then use finite-difference coefficients (or stencils for more than one dimension) \cite{levequeFiniteDifferenceMethods2007,strikwerdaFiniteDifferenceSchemes2004} that linearly aggregate the values of neighboring nodes. 
In our setting (more formally, for hyperbolic PDEs), %
this means that the value of the temporal derivative at a node's position depends only on a local neighborhood around it and its spatial coordinate. 
Again for $\Omega = [0,L]$, this means that for sufficiently small $\Delta t$, one can approximate 
$$
    \overline v_i(t + \Delta t) \quad\text{from}\quad \overline v_{i-\ell}(t), \dots , \overline v_{i+\ell}(t),
$$
that is, from a few ``neighboring'' values of its previous state, where $\ell$ is usually small. (In the case of spatial inhomogeneities, one also has has a dependence on $i$.)
When it comes to \emph{learning} these PDEs, this is commonly exploited using CNN-based architectures \cite{RRL+22,WP24,long2018pde}. Our method is similar, but we explicitly train one small model that is applied everywhere.

\subsection*{Extreme learning machines}

We use extreme learning machines (ELMs) \cite{dingExtremeLearningMachine2014,huang2006extreme} to predict the flow. Simply put, ELMs are two-layered neural networks that are trained in a certain way: The weights of the first layer are chosen randomly (in our case, always uniformly from $-\sqrt k, \dots, \sqrt k$, where $k = 1/\lin$ and $\lin$ is the number of inputs), and only the second bias-free layer is trained. In contrast to the usual gradient-based methods, this is done by directly solving a least-squares problem. Given a pair of random variables $(\zf,\zf^+)$ modeling the input-output relationship with values in $\R^{\lin}$ and $\R^{\lout}$ respectively, ELMs define a mapping $E : \R^{\lin} \rightarrow \R^{\lout}$ by solving
\begin{align}
\label{eq:elm-formulation}
    \min_{\Theta \in \R^{\lout \times \lhidden}} \E[\lVert E(\zf)  - \zf^+ \rVert^2_2], \quad \quad E(\zf) = \Theta \phi(\zf),
\end{align}
where $\lhidden$ is the hidden dimension and $\phi : \R^\lin \rightarrow \R^{\lhidden}$ composes the first layer with the activation function. Therefore, only the linear output layer, parameterized by $\Theta$, is learned. Setting the derivative of the loss with respect to $\Theta$ to zero, we may express the optimal solution as
\begin{equation}
    \Theta^* = \E[\zf^+ \phi(\zf)^T] \E[\phi(\zf) \phi(\zf)^T]^{-1} = CD^{-1}.
    \label{eq:optimal_Theta}
\end{equation}
The $\lout \times \lhidden$ and $\lhidden \times \lhidden$-matrices $C$ and $D$ are estimated---not computed exactly---by averaging finitely many known input-output pairs. %
Note that reservoir computers (e.g., \cite{pandeyReservoirComputingModel2020}) are trained very similarly; they are the recurrent pendant to the ELM's feed-forward  architecture.
The mapping $\phi$ can be seen as a random $\lhidden$-dimensional dictionary of observables. That is, it takes the data and maps it into a potentially higher dimension, making it easier to learn a linear mapping from embedded input to output.
This training procedure is shared by, or at least similar to, many other approaches such as sparse-identification of nonlinear dynamics (SINDy) \cite{bruntonDiscoveringGoverningEquations2016}, radial-basis function networks or extended dynamic mode decomposition \cite{williams2015data}. Due to the random initialization, an advantage of ELMs is that the hyperparameter tuning is omitted up to choosing an activation function and an initialization strategy for the first layer. Its advantage over gradient-based methods is that the optimal solution is highly precise and recovered almost instantaneously.

\begin{figure}[ht]
    \begin{center}
    \scalebox{0.8}{
        \pgfmathsetmacro{\ss}{0.4}
        \pgfmathsetmacro{\offset}{3}
        \pgfmathsetmacro{\rwidth}{3}
        \pgfmathsetmacro{\roffset}{3}
        \begin{tikzpicture}
            \draw [thin, gray] (0,0) -- (0,7*\ss);
            \draw [thin, gray] (\ss,0) -- (\ss,7*\ss);
            \foreach \i in {0,...,7} {
                \draw [thin, gray] (0,\i*\ss) -- (\ss,\i*\ss);
            }

            \draw [fill=cyan] (0, 2*\ss) rectangle ++(\ss,3*\ss) ;
            \draw [very thick, black] (0,\ss) -- (0,6*\ss);
            \draw [very thick, black] (\ss,\ss) -- (\ss,6*\ss);
            \foreach \i in {1,...,6} {
                \draw [very thick, black] (0,\i*\ss) -- (\ss,\i*\ss);
            }
            
            \draw [thin, gray] (\offset,0) -- (\offset,7*\ss);
            \draw [thin, gray] (\offset+\ss,0) -- (\offset+\ss,7*\ss);
            \foreach \i in {0,...,7} {
                \draw [thin, gray] (\offset,\i*\ss) -- (\offset+\ss,\i*\ss);
            }

            \draw [fill=cyan] (\offset, 2*\ss) rectangle ++(\ss,3*\ss) ;
            \draw [very thick, black] (\offset,2*\ss) -- (\offset,5*\ss);
            \draw [very thick, black] (\offset+\ss,2*\ss) -- (\offset+\ss,5*\ss);
            \foreach \i in {2,...,5} {
                \draw [very thick, black] (\offset,\i*\ss) -- (\offset+\ss,\i*\ss);
            }

            \draw[thick] (\roffset*\ss, 1.5*\ss) rectangle 
                (\roffset*\ss+\rwidth*\ss,5.5*\ss) node[midway] {ELM};

            \draw [thick, black] (\ss,\ss) -- (\roffset*\ss, 1.5*\ss);
            \draw [thick, black] (\ss,6*\ss) -- (\roffset*\ss, 5.5*\ss);
            \draw [thick, black] (\rwidth*\ss+\roffset*\ss,1.5*\ss) -- (\offset,2*\ss);
            \draw [thick, black] (\rwidth*\ss+\roffset*\ss,5.5*\ss) -- (\offset,5*\ss);

            \draw node[anchor=east] at (-1*\ss,6*\ss+0.5*\ss) {$i - 2$};
            \draw node[anchor=east] at (-1*\ss,5*\ss+0.5*\ss) {$i - 1$};
            \draw node[anchor=east] at (-1*\ss,4*\ss+0.5*\ss) {$i$};
            \draw node[anchor=east] at (-1*\ss,3*\ss+0.5*\ss) {$i + 1$};
            \draw node[anchor=east] at (-1*\ss,2*\ss+0.5*\ss) {$i + 2$};
            \draw node[anchor=east] at (-1*\ss,1*\ss+0.5*\ss) {$i + 3$};
            \draw node[anchor=east] at (-1*\ss,0*\ss+0.5*\ss) {$i + 4$};
            
            \draw node[anchor=west] at (\offset+1.5*\ss,4*\ss+0.5*\ss) {$i$};
            \draw node[anchor=west] at (\offset+1.5*\ss,3*\ss+0.5*\ss) {$i + 1$};
            \draw node[anchor=west] at (\offset+1.5*\ss,2*\ss+0.5*\ss) {$i + 2$};

            \draw (\roffset*\ss+\rwidth*0.5*\ss, +8*\ss) node {$p(i \Delta x)$};
            \draw[->] (\roffset*\ss+\rwidth*0.5*\ss, +8*\ss-0.3) -- (\roffset*\ss+\rwidth*0.5*\ss, 5.5*\ss);

            \draw node at (0.5*\ss,-\ss) { $\overline v(t) \in \R^m$ };
            \draw node at (\offset+0.5*\ss,-\ss) { $\overline v(t+\Delta t) \in \R^m$};

        \end{tikzpicture}}
    \end{center}
    \vspace{-2em}
    \caption{Sliding-window approach with an extent of $\ell = 1$ and a stride of $s = 3$.}
    \label{fig:sliding_window_1d}
\end{figure}

\section{Extreme learning machines for PDEs} 
\label{sec:EELMs}

When it comes to predicting PDEs, a simple and frequently used test-bed is the one-dimensional KS equation, which comes in a spatially inhomogeneous \eqref{eq:KS_1d_spatial} and homogeneous variant \eqref{eq:KS_1d}. With regards to this setting, most approaches are based on recurrent neural networks or more specifically reservoir computing~\cite{pathakModelFreePredictionLarge2018,rafayelyanLargeScaleOpticalReservoir2020,vlachasBackpropagationAlgorithmsReservoir2020}.
Falling into the latter category, the method in \cite{pathakModelFreePredictionLarge2018} trains one (in case of spatial homogeneity) or multiple (in case of spatial inhomogeneity) reservoirs that handle different parts of the space.
We employ a similar approach but always train a single ELM that is applied everywhere. Essentially, our implementation shifts a window over the coordinate space and predicts the PDE's state at the succeeding and smaller window, with an optional ``positional encoding'' that tells the ELM its current location.
Besides demonstrating our approach on \cref{eq:KS_1d,eq:KS_1d_spatial}, we show that complex systems such as \eqref{eq:KS_Nd} and \eqref{eq:CH_Nd} in two spatial dimensions are also learnable; moreover, for the case of the two-dimensional KS equation, we show how one can exploit multiple inherent symmetries. 

\subsection{Model architecture and training}\label{sec:training}

We describe our method for the one-dimensional case only, since its extension to two or more dimensions is straight forward.
In this case, the ELM is a function of the form $E : \R^{2\ell + s} \rightarrow \R^s$, where $\ell$ is the \emph{extent} and $s$ the \emph{stride}. 
Intuitively, $s$ defines how many values are predicted simultaneously, while $\ell$ defines the amount of ``extra information'' that is taken into account left and right; see \Cref{fig:sliding_window_1d} for an illustration.
In two dimensions, the window is a matrix, and the value at a certain location is approximated by a function accepting a two-dimensional window as its input.
Taking the (periodic) boundary conditions into account, one can then view the ELM %
as an approximation to the flow $\overline v(t) \mapsto \overline v(t+ \Delta t)$ by ``gluing'' copies of the same model together.
The training is done following \eqref{eq:elm-formulation} and \eqref{eq:optimal_Theta}, where $\bm z$ and $\bm z^+$ are randomly-sampled input-output windows:
Taking states from a single trajectory by sampling a random time $\bm t$ and a random index $\bm i$, we define $\bm z$ and $\bm z^+$ by:
\begin{align*}
    \bm z &= (\quad \overline v_{\bm i-\ell} (\bm t), \dots, \overline v_{\bm i+s-1+\ell}(\bm t) \quad), \quad\quad &&\text{(with values in $\R^{2\ell +s}$)} \\
    \bm z^+ &= ( \quad \overline v_{\bm i}(\bm t + \Delta t), \dots, \overline v_{\bm i+s-1}(\bm t + \Delta t) \quad ).
    \quad\quad &&\text{(with values in $\R^s$)}
\end{align*}
For an explicit dependence on the spatial coordinate as in \eqref{eq:KS_1d_spatial}, one can include information about the model's current location in the mapping, which we do by applying a positional encoding and concatenating it with the model's input.
In that case, we found it best to work with a Transformer-inspired binary-like encoding of the position:
\begin{align*}
    p(x) = (\cos(\pi x/L),\cos(2\pi x/L),\cos(4\pi x/L),\dots,\cos(2^k \pi x/L)),
\end{align*}
with the idea that the model could ``branch'' on the components of $p(x)$ whenever it needs to know its current location. The positional encoding is then concatenated with the input window (in this case, $E:  \R^{2 \ell + s + k + 1} \rightarrow \R^s$).

\begin{table}[t]
    \centering
    \def\arraystretch{1.1}
    \setlength{\tabcolsep}{0.22em}
    \begin{tabular}{l|c|c|c|c}
        Parameter & hom. KS \eqref{eq:KS_1d} & inhom. KS \eqref{eq:KS_1d_spatial} & 2d KS \eqref{eq:KS_Nd} & 2d CH \eqref{eq:CH_Nd} \\
        \hline 
        domain size $L$ & 200 & 200 & $60\pi$ & 100 \\
        grid size $m$ & 512 & 512 & 256 & 512 \\
        initial state $\overline v(0)$ & on attractor & on attractor & on attractor & uniformly $-1\dots 1$ \\
        time discretization $\Delta t$ & 0.05 & 0.05 & 0.05 & 0.05 \\
        additional parameters & -- & $\lambda = 50, \mu = 0.05$ & -- & $\gamma = 0.5$ \\
        \hline
        extent $\ell$ & 7 & 7 & 2 & 4 \\
        stride $s$ & 4 & 4 & 4 & 4 \\
        hidden units $\lhidden$ & 150 & 150 & 600 & 500 \\
        noise std. $\sigma$ & $10^{-4}$ & $10^{-4}$ & $10^{-4}$ & $10^{-3}$ \\
        training samples & 20 & 200 & 1 & 1 \\
        positional encoding $k$ & 0 & 3 & 0 & 0 \\
    \end{tabular}
    \vspace{1em}
    \caption{Experimental setup of both simulation and ELM for different PDEs.}
    \label{tab:experimental-parameters}
    \vspace{-2em}
\end{table}

\subsection{Experiments}

In this section, we conduct multiple experiments on different PDEs. Every experiment illuminates a different aspect of the proposed method: 1), for the homogeneous one-dimensional KS equation, we train multiple models and demonstrate their robustness for different initializations. 2), for its inhomogeneous variant, we show that positional encodings allow the model to recover typical patterns to a higher degree than in their absence. 3), we show that a \emph{single} input-output pair suffices for training an ELM on the two-dimensional CH and KS equations. And 4), for the two-dimensional KS equation, we show that multiple system-inherent symmetries can be exploited, giving rise to \emph{equivariant} ELMs.

For the general setup, we always use the softplus nonlinearity (which worked best for all tested nonlinearities), apply min-max normalization, and train the ELM by adding zero-mean Gaussian noise on the input (this is essential for the model's stability). The remaining parameters are shown in \Cref{tab:experimental-parameters}---these values were chosen rather ad-hoc, but there is a general trade-off between speed and precision: If speed is a priority, one should use a small hidden dimension, a small extent, a large stride (so that more values are predicted in parallel). For precision, one should do the opposite. %

Training is done as described in the previous section; we iteratively pick a random snapshot and a window of input-dimensional size, compute its embedding and the target window, and improve the estimates of $C$ and $D$ from \eqref{eq:optimal_Theta} using running averages. 
To quantify prediction quality, we use the squared error between simulated ($\overline v$) and predicted ($\widehat v$) state that is then normalized by the performance of a ``baseline'' model that predicts $\overline v(t)$'s average value (the computed quantity is then proportional to $\overline v(t)$'s variance when indices are sampled randomly). We obtain the relative squared error (RSE) as a function of time:
$$
    \mathrm{RSE}(t) = 100\% \cdot \frac{\lVert \overline v(t)-\widehat v (t) \rVert^2_2}{\lVert \overline v(t)- \frac 1 m \sum_{i=0}^{m-1} \overline v_i(t)] \rVert^2_2}.
$$

\begin{figure}[t]
    \begin{center}
        \includegraphics[width=0.7\textwidth]{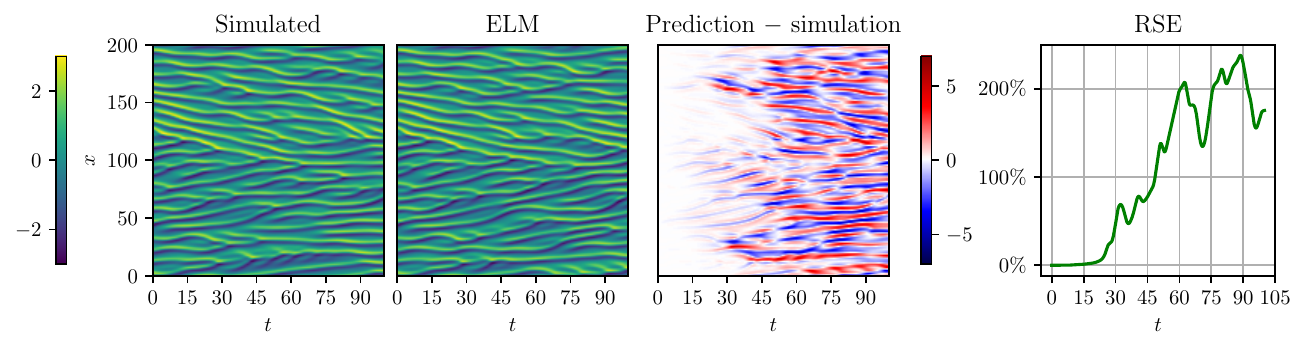}
        \includegraphics[width=0.29\textwidth]{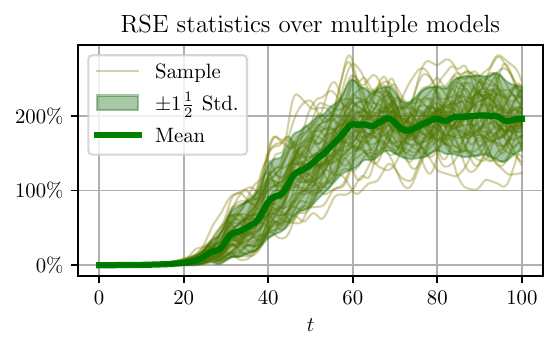}
    \end{center}
    \vspace{-1em}
    \caption{Exemplary simulation of the homogeneous KS system compared to an ELM's prediction (left) and statistics over multiple ($\sim 50$) models (right).}
    \label{fig:KS_1d_homogeneous}
    \vspace{-1em}
\end{figure}

\paragraph{Model statistics (\cref{eq:KS_1d}).}

For the homogeneous KS equation, we show the simulated trajectory together with an ELM-based prediction and their deviation from each other on the left side of \Cref{fig:KS_1d_homogeneous}. To showcase the robustness of this approach, we have trained multiple models for different initial weights and plotted their deviation from the target trajectory on the right side of \Cref{fig:KS_1d_homogeneous}. In all cases, the ELM accurately predicts the first few time steps. The error increases with time, which is expected for chaotic systems, and converges to an average of $200\%$. %
The obtained results are similar to that of reservoir-based methods, compare for example the right side of \Cref{fig:KS_1d_homogeneous} with Figure 5 in \cite{pandeyReservoirComputingModel2020}.

\begin{figure}[t]
    \begin{center}
        \includegraphics[width=0.8\textwidth]{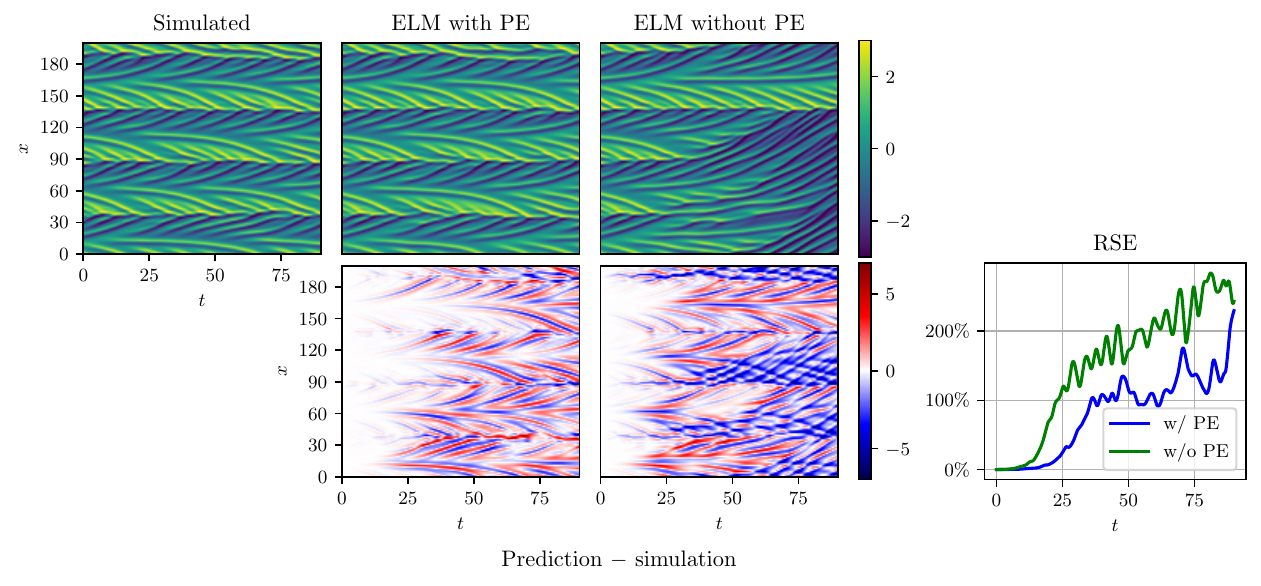}
    \end{center}
    \vspace{-2em}
    \caption{Simulation of the inhomogeneous KS system compared to the ELM's prediction with and without positional encoding (PE).} %
    \label{fig:KS_1d_inhomogeneous}
\end{figure}

\begin{figure}[t]
    \begin{center}
        \includegraphics[width=.75\textwidth]{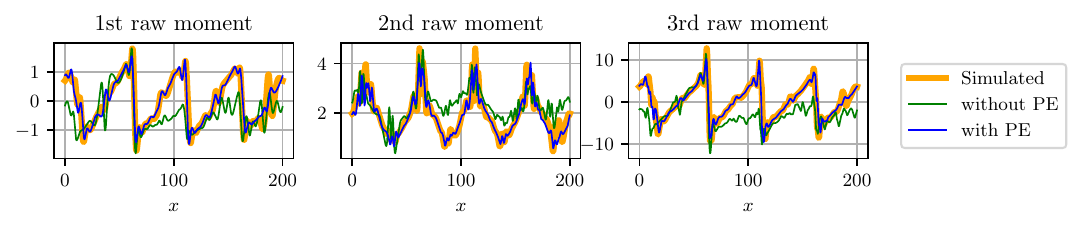}
    \end{center}
    \vspace{-1em}
    \caption{Statistics from the inhomogeneous KS system. The $k$-th raw moment of a trajectory $\overline v$ (simulated or predicted by one of the ELMs) is $\E[\overline v(\bm t)^k]$, where $\bm t$ is sampled uniformly from the time interval. Similar flow statistics were used by \cite{srinivasan2019predictions}.}
    \label{fig:KS_1d_inhomogeneous_statistics}
\end{figure}

\paragraph{Using positional encodings (\cref{eq:KS_1d_spatial}).}
For the inhomogeneous KS equation, \Cref{fig:KS_1d_inhomogeneous} shows that positional encodings significantly improve the prediction quality and allow the ELM to recover typical patterns. In contrast, if there are no positional encodings, then the ``peaks'' eventually collapse. This is also supported by \Cref{fig:KS_1d_inhomogeneous_statistics}, in which we estimate the first three raw moments of the predicted trajectories as functions of $x$ by averaging over time. (Intuitively, these moments are related to the mean, variance and skewness of the PDE state at some location $x$.) Comparing the resulting graphs, the statistics generated by the ELM with positional encodings closely resemble those of the simulated trajectory.

\begin{figure}[t]
    \begin{center}
        \includegraphics[width=0.74\textwidth]{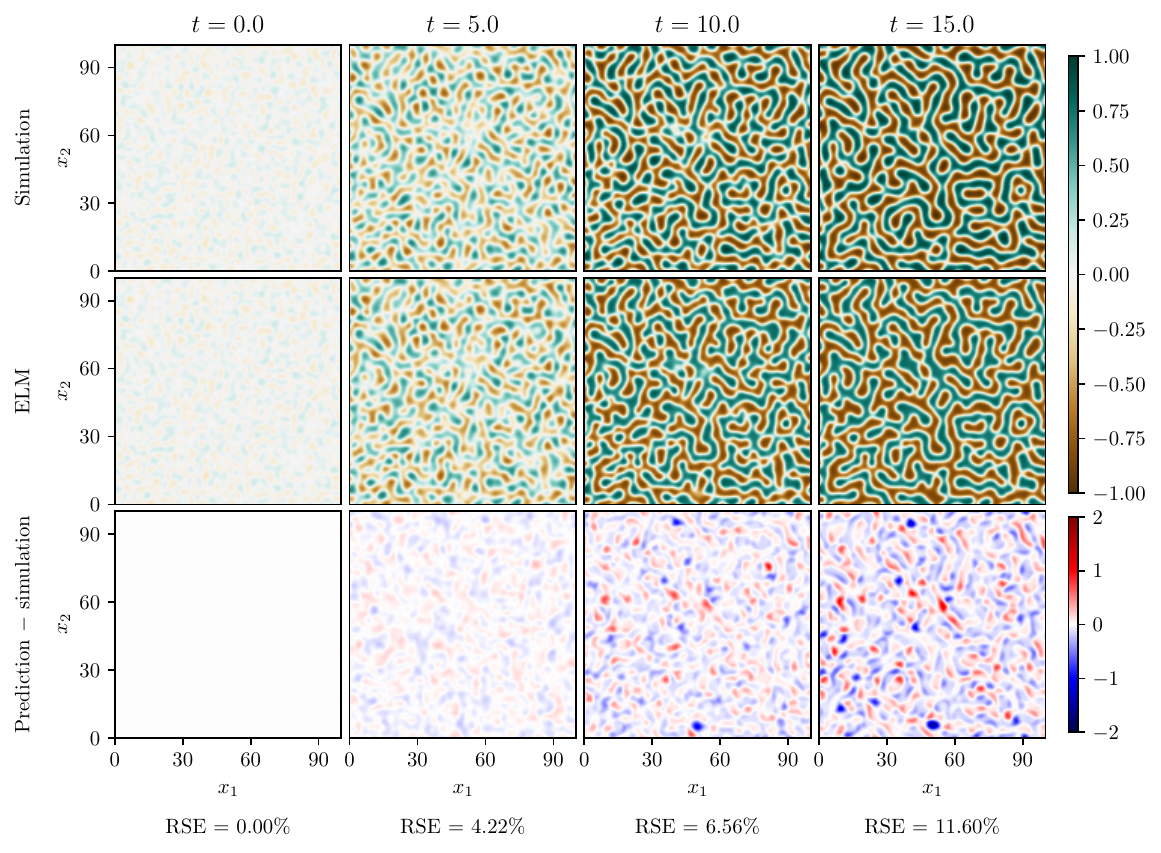}
        \includegraphics[width=0.25\textwidth]{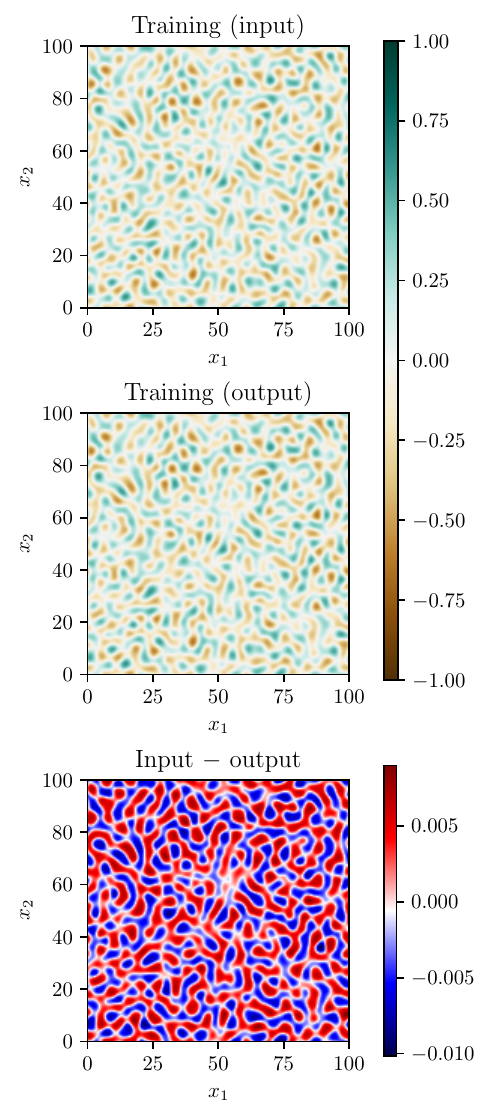}
    \end{center}
    \vspace{-2em}
    \caption{Simulation of the two-dimensional Cahn-Hilliard equation compared to the ELM's prediction (left), together with the training data (right).}
    \label{fig:ch2d}
\end{figure}

\paragraph{Learning from a single sample (\cref{eq:CH_Nd,eq:KS_Nd}).}
The Cahn-Hilliard equation describes phase separation processes in binary mixtures, where two components are separating over time due to differences in concentration. As the equation evolves, the initially small and numerous domains merge or disappear, forming areas where one component is more concentrated than the other, see the top row in \Cref{fig:ch2d}. %
It is possible to learn this system using an ELM, and, surprisingly, a single input-output pair suffices to train it. That is, even though the ELM was trained on a single frame (right column in \Cref{fig:ch2d}), it is still able to accurately predict the PDE's evolution for qualitatively different states (center row of \Cref{fig:ch2d}).
For the KS equation, note that variant \eqref{eq:KS_Nd} is different from \eqref{eq:KS_1d} and \eqref{eq:KS_1d_spatial}. In fact, in contrast to the latter two equations, it is unstable in the sense that the state's mean diverges over time (e.g., \cite{kalogirou2015depth}). To obtain consistent behavior, one can solve the equation by zeroing out the state's mean, apply one or multiple integration steps, and add the previously substracted mean back in again. For the ELM-based surrogate, we make this part of the pre- and post-processing procedure, so that it always receives zero-mean states as its input. The ELM's prediction together with the simulated trajectory and their deviation is shown in \Cref{fig:ks2d}. 
Again, a single input-output pair suffices for training.

\begin{figure}[t]
    \begin{center}
        \includegraphics[width=0.74\textwidth]{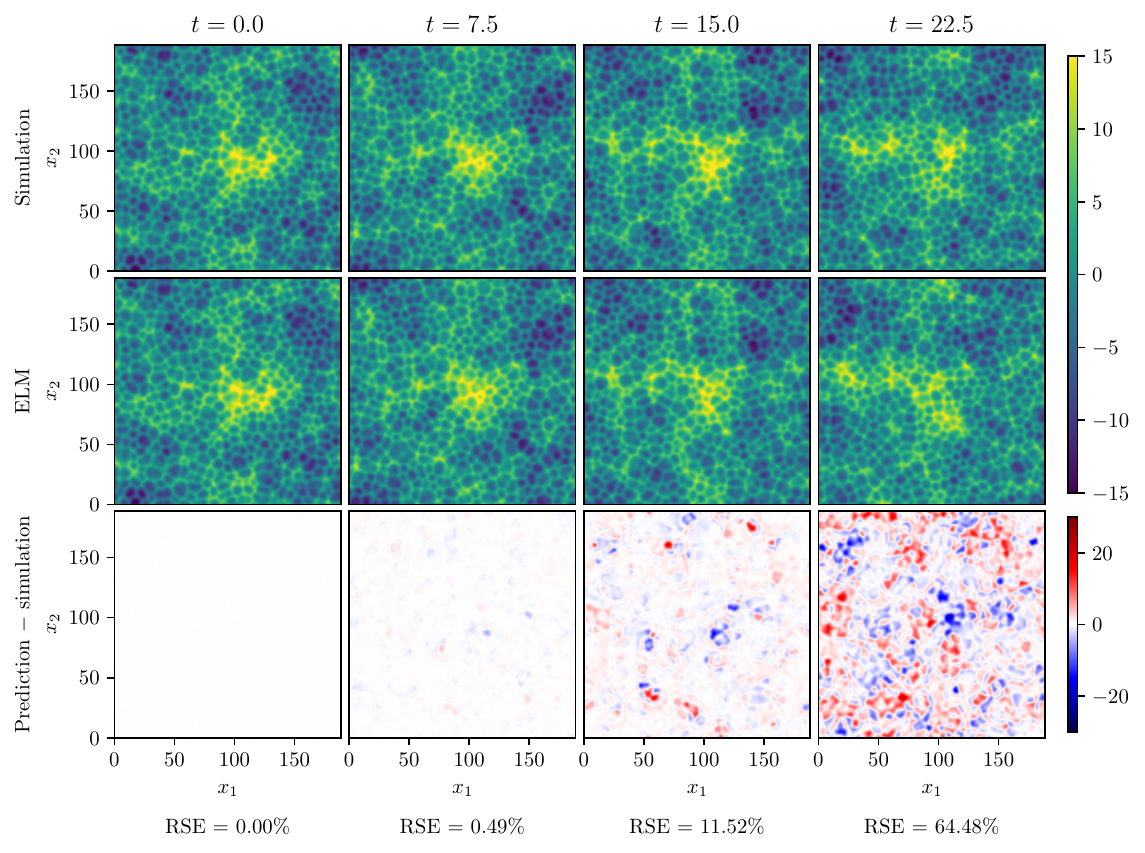}
        \includegraphics[width=0.25\textwidth]{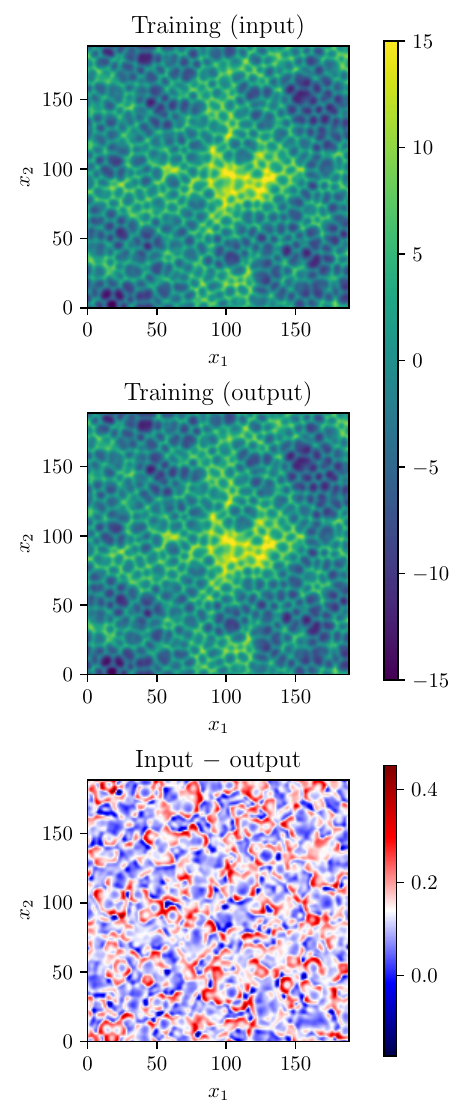}
    \end{center}
    \vspace{-2em}
    \caption{Simulation of the two-dimensional Kuramoto-Sivashinsky equation compared to the ELM's prediction (left), together with the training data (right).}
    \label{fig:ks2d}
\end{figure}

\subsection{Exploiting more symmetries and enforcing equivariance}
\label{sec:AdditionalSymmetries}

\begin{figure}[t]
    \centering
    \begin{subfigure}[t]{0.4\textwidth}
        \centering
        \begin{tikzpicture}
            \node[inner sep=0pt] at (0,1.5) 
                {\includegraphics[width=20pt]{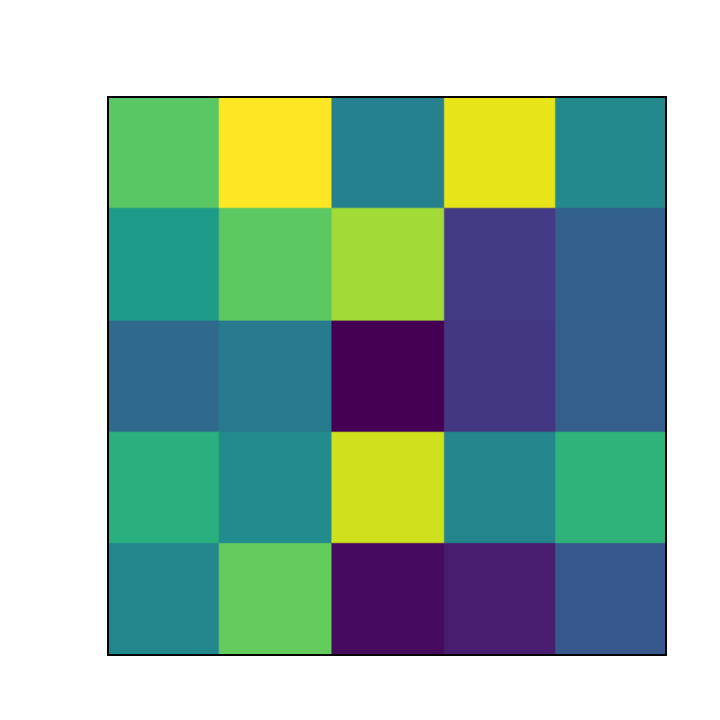}};
            
            \node[inner sep=0pt] at (2,1.5) 
                {\includegraphics[width=14pt]{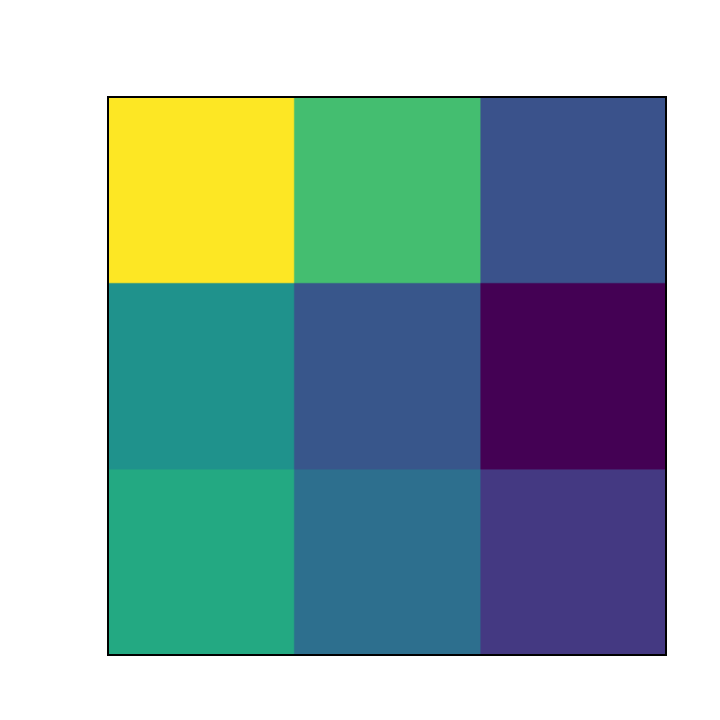}};
                
            \node[inner sep=0pt] at (0,0) 
                {\includegraphics[width=20pt]{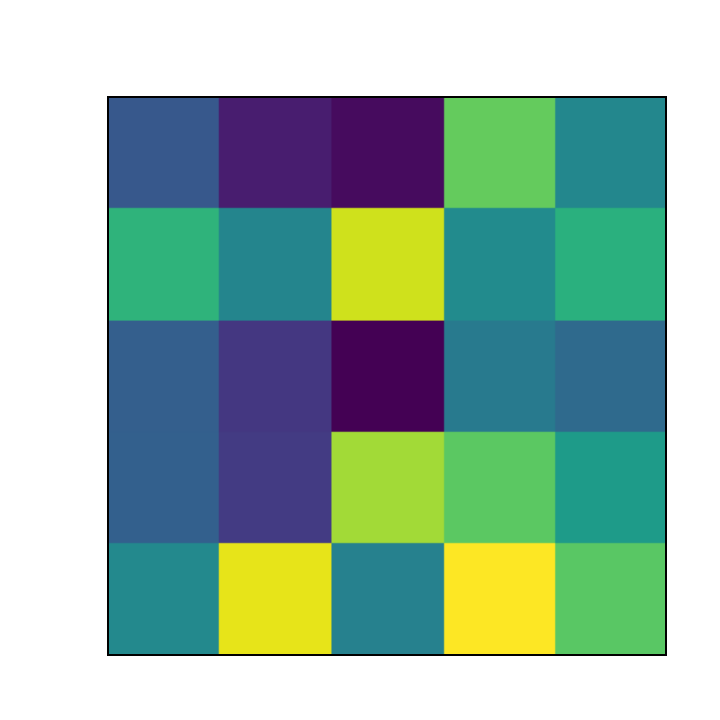}};
                
            \node[inner sep=0pt] at (2,0) 
                {\includegraphics[width=14pt]{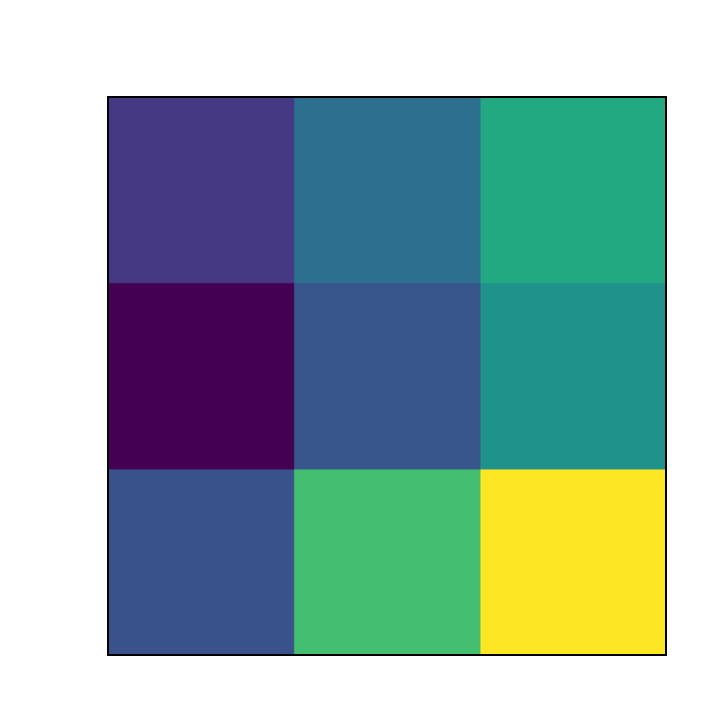}};
    
            \draw [<-] (-0.05,0.4) -- (-0.05,1.1) node[midway, left, shift=({-0.05,0.1})] {$R_i$};
            \draw [->] (0.05,0.4) -- (0.05,1.1) node[midway, right, shift=({+0.05,0.1})] {$R_i^{-1}$};
            
            \draw [<-] (1.95,0.4) -- (1.95,1.1) node[midway, left, shift=({-0.05,0.1})] {$R_i$};
            \draw [->] (2.05,0.4) -- (2.05,1.1) node[midway, right, shift=({+0.05,0.1})] {$R_i^{-1}$};
    
            \draw[->] (0.4, 1.5) -- (1.7, 1.5) node[above, midway] {$t+\Delta t$};
            \draw[->] (0.4, 0) -- (1.7, 0) node[below, midway] {$t+\Delta t$};

            \node at (1,-1) {$i=1,\dots,8$};
        \end{tikzpicture}
        \caption{Transforming an input window, evolving it forward and transforming it back yields the same result. Both top and bottom row contain valid input-output pairs.}
        \label{fig:window-equivariance}
    \end{subfigure}
    \hfill
    \begin{subfigure}[t]{0.5\textwidth}
        \scalebox{0.7}{
            \begin{tikzpicture}
                \draw [->] (0,5) -- (-2.75,4) node[midway, left, shift=({-0.3,0.1})] {$R_1$};
                \draw [->] (0,5) -- (-1.75,4);
                \draw [->] (0,5) -- (-1   ,4);
                \draw [->] (0,5) -- (-.3  ,4);
                \draw [->] (0,5) -- (+.3  ,4);
                \draw [->] (0,5) -- (+1   ,4);
                \draw [->] (0,5) -- (+1.75,4);
                \draw [->] (0,5) -- (+2.75,4) node[midway, right, shift=({+0.2,0.1})] {$R_8$};
                \node[inner sep=0pt] at (0,5) 
                    {\includegraphics[width=20pt]{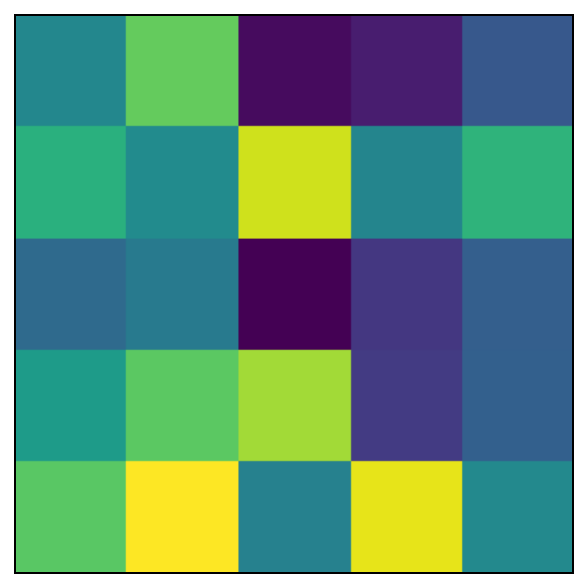}};

                \foreach \i in {0,...,7} {
                    \draw [->] (-100pt + 12.5pt + \i*200pt/8, 3.5) -- ++(0, -1) ;
                }
                \draw node at (-110pt, 2.75) {$E$};
                \draw node at (-75pt, 2.75) {\dots};
                \draw node at (+110pt, 2.75) {$E$};
                
                \foreach \i in {0,...,7} {
                    \draw [->] (-100pt + 12.5pt + \i*200pt/8, 2) -- ++(0, -1);
                }
                \draw node at (-110pt, 1.3) {$R_1^{-1}$};
                \draw node at (-75pt, 1.3) {\dots};
                \draw node at (+110pt, 1.3) {$R_8^{-1}$};
                
                \draw [<-] (-16pt,-0.9) -- (-2.75, 0);
                \draw [<-] (-12pt,-0.75) -- (-1.75, 0);
                \draw [<-] (-8pt ,-0.7) -- (-1   , 0);
                \draw [<-] (-4pt ,-0.65) -- (-.3  , 0);
                \draw [<-] ( 4pt ,-0.65) -- (+.3  , 0);
                \draw [<-] ( 8pt ,-0.7) -- (+1   , 0);
                \draw [<-] ( 12pt,-0.75) -- (+1.75, 0);
                \draw [<-] ( 16pt,-0.9) -- (+2.75, 0);

                \draw node at (+75pt, -0.6) {$\frac 1 8 \times \sum$};

                \node[inner sep=0pt] at (0,3.5) 
                    {\includegraphics[width=200pt]{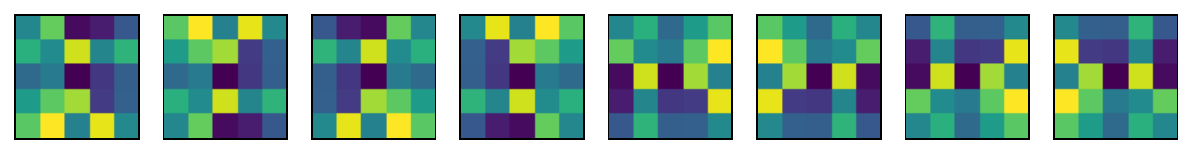}};

                \node[inner sep=0pt] at (0,2) 
                    {\includegraphics[width=192pt]{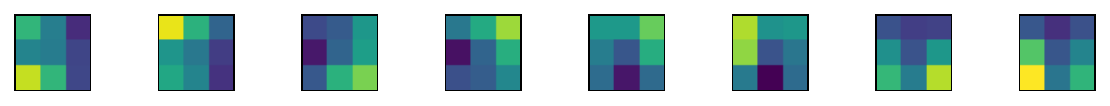}};
                
                \node[inner sep=0pt] at (0,.5) 
                    {\includegraphics[width=192pt]{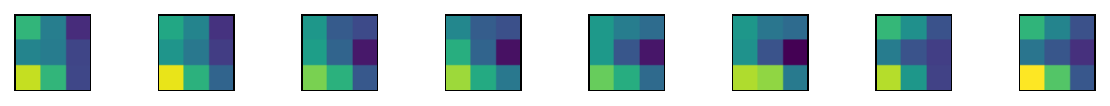}};
                
                \node[inner sep=0pt] at (0,-1) 
                    {\includegraphics[width=15pt]{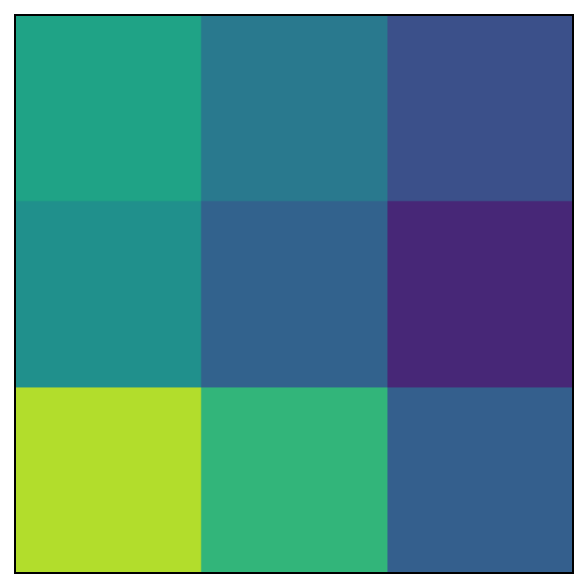}};
            \end{tikzpicture}
        }
        \caption{Equivariance can be enforced by transforming the input using each group element, applying the ELM, inverting each transformation and averaging the outputs.}
        \label{fig:averaging-equivariance}
    \end{subfigure}
    \caption{Schemes describing how the system's equivariance is transferred and can be used on a window level. (Stride $s = 3$ and extend $\ell = 1$.)}
\end{figure}

\begin{figure}[t]
    \begin{center}
        \includegraphics[width=0.75\textwidth]{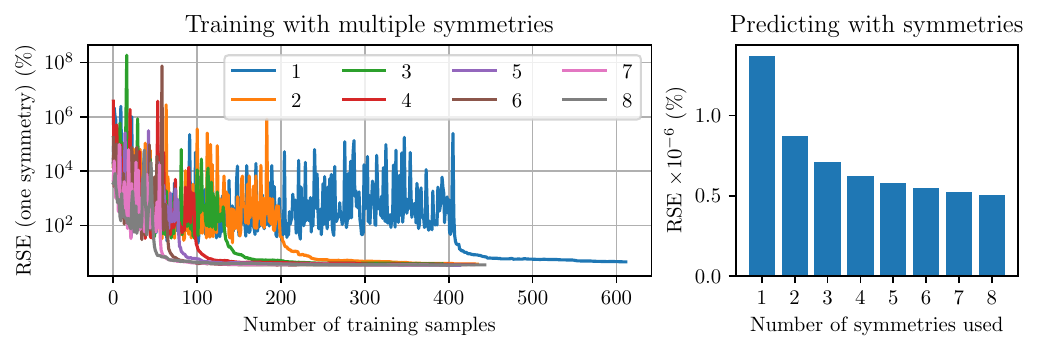}
    \end{center}
    \vspace{-1em}
    \caption{Symmetries can be used to increase data efficiency (left) and to improve the prediction quality (right). In the left figure, a sample is one input-output window pair.}
    \label{fig:exploiting-symmetries}
\end{figure}

The KS system \eqref{eq:KS_Nd} for two (and even arbitrary) spatial dimensions admits many symmetries that can be exploited. In particular, it is equivariant with respect to the \emph{Euclidean group} $E(2)$ that contains functions of the form $R : x \mapsto A x + b$, where $x, b \in \R^2$ and $A \in \R^{2 \times 2}$ is an orthogonal matrix, together with function composition as the group operation. Composition of $R$ with $v$, i.e., $(v, R) \mapsto v \circ R$, defines a right group action on $\Fc$.
To see that \emph{equivariance} of \eqref{eq:KS_Nd} holds, i.e., that
\begin{equation}
    \Phi(v \circ R) = \Phi(v) \circ R,
    \label{eq:invariance-solution}
\end{equation}
one can first argue that $\Laplace$, $\Laplace^2$ and $| \nabla (\cdot)|$ are equivariant (appendix, \Cref{thr:invariance}) so that, consequently, equivariance of the temporal derivative is established. %
This argument extends to the PDE's solution (appendix, \Cref{thr:invariance-solution}).

Since we are working with small windows on finite grids, we restrict ourselves to the symmetries of the square---a subgroup of $E(2)$---defined in our case by appropriate reflections or rotations of windows 
across their centers. Overall, there are eight group elements (identity, rotations of 90, 180 and 270 degrees, and finally reflections along the horizontal, vertical and two diagonal axes), that act on the windows (see \Cref{fig:window-equivariance}).
We exploit this in two ways: First, we add new artificial samples by letting some or all group elements act on known input-output pairs (see \Cref{fig:window-equivariance}).
This decreases the overall number of real samples needed. Second, we use multiple symmetries of one sample to generate a single output and therefore improve the prediction quality. That is, we predict an output-window by averaging over some or all group elements as shown in \Cref{fig:averaging-equivariance}. If we include all eight, this enforces full equivariance.
\Cref{fig:exploiting-symmetries} shows that both approaches work well. Adding artificial training samples significantly reduces the overall number of ``real'' samples needed, and using multiple symmetries for prediction improves the model's accuracy. Notice that both steps can be executed separately: It is possible to train a model on multiple symmetries but to make predictions using less, and vice versa.

\subsubsection*{Conclusion.}
An ELM-based approach appears to be highly promising for the task of modeling PDEs. This is due to several reasons. 
First, since the trained model is very small, it has only low data requirements. %
Second, it is not recurrent like previous approaches, a fact that significantly simplifies the setup including training process and prediction.
Third, %
it is possible to exploit additional system-inherent symmetries to decrease the number of necessary training samples and to improve the prediction quality.

\paragraph{Acknowledgement.}
HH and SP acknowledge financial support by the project ``SAIL: SustAInable Life-cycle of Intelligent Socio-Technical Systems'' (Grant ID NW21-059D), which is funded by the program ``Netzwerke 2021'' of the Ministry of Culture and Science of the State of Northrhine Westphalia, Germany.

\begin{appendix}

\section{Appendix}

Let $E(n)$ be the Euclidean group of dimension $n$, let $R \in E(2)$ and denote $R(x) = A x + d$, where $A \in \R^{n \times n}$ is orthogonal and $d \in \R^n$ arbitrary. For $v : \R^n \rightarrow \R$ and $\ifat \in J^m$, denote $\partial_{\ifat}^m v := \partial_{\ifat_1,\dots,\ifat_m}^m v$ and $|\nabla v|(x) := |\nabla v(x)|$.

\begin{lemma}
    If $v : \R^n \rightarrow \R$ is $m$-times differentiable and $\ifat \in J^m$, then
    \begin{equation}
        \partial^m_{\ifat} ( v \circ R )(x) 
        = \sum_{\kfat \in J^m} (\partial^m_{\kfat} v \circ R)(x) \prod_{s=1}^m A_{\kfat_s,\ifat_s}.
        \label{eq:thr-1}
    \end{equation}
\end{lemma}
\begin{proof}
    This is done by induction on $m$. For $m = 1$, notice that 
    \begin{align*}
        \partial_i (v \circ R)(x) 
        &= \textstyle \lim_{\delta \rightarrow^+ 0} \textstyle\frac 1 \delta (v(R(x+\delta {\bm 1}_i))-v(R(x)) ) \\
        &= \textstyle \lim_{\delta \rightarrow^+ 0} \textstyle\frac 1 \delta (v((Ax+d)+\delta A {\bm 1}_i)-v(Ax+d) ) \\
        &= \textstyle \nabla_{A{\bm 1}_i} v(Ax+d) 
        =  (A{\bm 1}_i)^T \nabla v(Ax+d)
        =  \sum_{j=1}^n A_{ji} (\partial_j v \circ R)(x),
    \end{align*}
    where $\bm{1}_i \in \R^n$ is the indicator vector for index $i$ and $\nabla_{(\cdot)}$ the directional derivative.
    For $m \geq 1$, we have inductively for $\ifat \in J^{m+1}$, $\ifat' = (\ifat_1,\dots,\ifat_m)$, $i = \ifat_{m+1}$:
    \begin{align*}
        \partial_{\ifat}^{m+1} (v \circ R)(x) 
        &= \textstyle \partial_{i} \partial^{m}_{\ifat'} (v\circ R)(x) \\
        &= \textstyle \partial_{i} \sum_{\kfat \in J^m}^n (\partial^m_{\kfat} v \circ R)(x) \prod_{s=1}^m A_{\kfat_s,\ifat'_s}  \\
        &= \textstyle \sum_{\kfat \in J^m} \big( \sum_{k=1}^n A_{k, i} (\partial^{m+1}_{\kfat,k} v\circ R)(x)\big)\prod_{s=1}^m A_{\kfat_s,\ifat'_s} \\
        &= \textstyle \sum_{\kfat \in J^{m+1}}  (\partial^{m+1}_{\kfat} v \circ R) (x) \prod_{s=1}^{m+1} A_{\kfat_s,\ifat_s}.
    \end{align*}
\end{proof}

\newpage

\begin{lemma}
    \label{thr:invariance}
    If $v : \R^n \rightarrow \R$ is four times differentiable and $A A^T = I$, then
    \begin{align*}
        | \nabla (v \circ R) | = | \nabla v \circ R |, \quad
        \Laplace (v \circ R) = \Laplace v \circ R, \quad
        \Laplace^2 (v \circ R) = \Laplace^2 v \circ R.
    \end{align*}
\end{lemma}
\begin{proof}
    Notice that $A A^T = I$ implies $\sum_{i=1}^n  A_{k,i} A_{j,i} = 1$ if $k = j$ and $0$ otherwise. Denote this fact by $(*)$.
    Then:
    \begin{align*}
        |\nabla (v\circ R)(x) |^2 
        &= \textstyle \sum_{i=1}^n (\partial_i (v\circ R)(x))^2 \\
        &= \textstyle \sum_{i=1}^n (\sum_{j=1}^n A_{ji} (\partial_j v \circ R)(x))^2    
            & \eqref{eq:thr-1} \\
        &= \textstyle \sum_{i=1}^n \sum_{j_1,j_2=1}^n A_{j_1, i} A_{j_2,i} (\partial_{j_1} v\circ R)(x)(\partial_{j_2} v\circ R)(x) \\
        &= \textstyle \sum_{j_1,j_2=1}^n (\partial_{j_1} v \circ R)(x)(\partial_{j_2} v \circ R)(x) \sum_{i=1}^n A_{j_1, i} A_{j_2,i} \\
        &= \textstyle \sum_{j=1}^n  (\partial_{j} v \circ R)(x)^2 = | (\nabla v \circ R)(x)|^2. & (*) \\
        \Laplace (v \circ R)(x) 
        &= \textstyle \sum_{i=1}^n \partial_{i,i}^2 (v\circ R)(x) \\
        &= \textstyle \sum_{i=1}^n \sum_{j_1,j_2=1}^n A_{j_1,i} A_{j_2,i} (\partial^2_{j_1,j_2} v \circ R)(x) 
            & \eqref{eq:thr-1} \\
        &= \textstyle \sum_{j_1,j_2=1}^n (\partial^2_{j_1,j_2} v \circ R)(x) \sum_{i=1}^n A_{j_1,i} A_{j_2,i} \\
        &= \textstyle \sum_{j=1}^n (\partial^2_{j,j} v \circ R)(x) = (\Laplace v \circ R)(x). 
            & (*) \\
        \Laplace^2 (v \circ R)(x) 
        &= \textstyle \sum_{i=1}^n \sum_{j=1}^n \partial_{i,i,j,j}^4 (v \circ R)(x) \\
        &= \textstyle \sum_{i=1}^n \sum_{j=1}^n \sum_{\kfat \in J^4} A_{\kfat_1,i} A_{\kfat_2,i} A_{\kfat_3,j} A_{\kfat_4,j} 
           (\partial^4_{\kfat} v \circ R)(x) 
            & \eqref{eq:thr-1} \\
        &= \textstyle \sum_{\kfat \in J^4} (\partial^4_{\kfat} v \circ R)(x) 
           \sum_{i=1}^n A_{\kfat_1,i} A_{\kfat_2,i}  \sum_{j=1}^n A_{\kfat_3,j} A_{\kfat_4,j} \\
        &= \textstyle \sum_{k_1=1}^n \sum_{k_3=1}^n (\partial^4_{k_1,k_1,k_3,k_3} v \circ R)(x) = (\Laplace^2 v \circ R)(x).
            & (*)
    \end{align*}
\end{proof}

\begin{lemma}
    \label{thr:invariance-solution}
    Define the explicit Euler method for $v \in \Fc$, $t \geq 0$, $\delta > 0$ by:
    \begin{align*}
        \Phi^{\delta}_t(v) := \begin{cases}
            v + t f(v) & t \leq \delta  \\
            \Phi^{\delta}_{t-\delta}(v) + \delta f( \Phi^\delta_{t-\delta}(v) ) & t > \delta,
        \end{cases}
    \end{align*}
    where $f : \Fc \rightarrow \Fc$ satisfies $f(v \circ R) = f(v) \circ R$ for all $v \in \Fc$, and assume that $\lim_{\delta \rightarrow^+ 0} \Phi^\delta_t(v) = \Phi_t(v)$ pointwise for $v \in \Fc$ and $t \geq 0$.
    Then 
    \begin{equation}
        \Phi_t(v \circ R) = \Phi_t(v) \circ R.
    \end{equation}
\end{lemma}
\begin{proof}
    We show $\Phi^\delta_t(v \circ R) = \Phi^\delta_t(v) \circ R$ by induction on $M = \lceil t / \delta \rceil$ . For $M \leq 1$, we have $\delta \geq t$. Using invariance of $f$ and linearity of the composition operator,
    \begin{align*}
        \Phi^\delta_t(v \circ R) = (v \circ R) + t f(v\circ R) 
        = (v + t f(v)) \circ R 
        = \Phi^\delta_t(v) \circ R.
    \end{align*}
    For the induction hypothesis, fix an $M \geq 1$ such that $\Phi^\delta_t(v \circ R) = \Phi^\delta_t(v) \circ R$ for all $\delta > 0, t \geq 0$ with $\lceil t / \delta \rceil \leq M$. 
    For the induction step, let $\lceil t / \delta \rceil = M + 1$. Then
    $$
        2 \leq M+1 = \lceil t / \delta \rceil < t/\delta + 1,
    $$
    and therefore $\delta < t$. Thus,
    \begin{align*}
        \Phi^\delta_t(v \circ R) 
        &= \Phi^\delta_{t-\delta}(v \circ R) + \delta f(\Phi^\delta_{t-\delta}(v \circ R)).
    \end{align*}
    Since 
    $$
        \lceil (t-\delta) / \delta \rceil = \lceil t/\delta - 1 \rceil = \lceil t / \delta \rceil - 1 = M, 
    $$
    we can use the induction hypothesis:
    $%
        \Phi^\delta_{t-\delta}(v \circ R) = \Phi^\delta_{t-\delta}(v)  \circ R.
    $ %
    Using invariance of $f$ and linearity of the composition operator, we obtain:
    \begin{align*}
        \Phi^\delta_t(v \circ R) 
        = (\Phi^\delta_{t-\delta}(v) + \delta f(\Phi^\delta_{t-\delta}(v))) \circ R 
        = \Phi^\delta_t(v) \circ R.
    \end{align*}
    This concludes the induction.
    Taking $\delta$ to zero, we finally have
    \begin{align*}
        \Phi_t(v\circ R) 
        &= \textstyle \lim_{\delta \rightarrow^+ 0} \Phi^\delta_t(v \circ R)  \\
        &= \textstyle \lim_{\delta \rightarrow^+ 0} (\Phi^\delta_t(v) \circ R) \\
        &= \textstyle (\lim_{\delta \rightarrow^+ 0} \Phi^\delta_t(v)) \circ R
        = \textstyle \Phi_t(v) \circ R.
    \end{align*}
\end{proof}

\end{appendix}

\bibliographystyle{plain} %
\bibliography{main} %

\end{document}